\documentclass{article}
\usepackage[utf8]{inputenc}
\usepackage[round]{natbib}
\usepackage[margin=1in]{geometry}
\usepackage{amsmath}
\usepackage{amssymb}
\usepackage{amsthm}
\usepackage{mathrsfs}
\usepackage[colorlinks=true,urlcolor=blue,linkcolor=blue,citecolor=blue]{hyperref}%
\usepackage{mathtools}
\usepackage{graphicx}
\usepackage{subcaption}
\usepackage{array}
\usepackage{pifont}
\usepackage{algorithm}
\usepackage{algpseudocode}
\def\Real{\mathop{\mathbb{R}}\nolimits}

\def\argmin{\mathop{\rm argmin}\nolimits}
\def\argmax{\mathop{\rm argmax}\nolimits}

\newcommand{\bb}{\boldsymbol{b}}
\newcommand{\bc}{\boldsymbol{c}}

\newcommand{\bv}{\boldsymbol{v}}

\newcommand{\bx}{\boldsymbol{x}}
\newcommand{\by}{\boldsymbol{y}}
\newcommand{\bz}{\boldsymbol{z}}
\newcommand{\bmu}{\boldsymbol{\mu}}

\newcommand{\bX}{\boldsymbol{X}}

\newcommand{\balpha}{\boldsymbol{\alpha}}

\newcommand{\btheta}{\boldsymbol{\theta}}

\newcommand{\bTheta}{\boldsymbol{\Theta}}

\newcommand{\Hi}{\mathcal{H}}
\newcommand{\M}{\mathcal{M}}
\newcommand{\G}{\mathcal{G}}
\newtheorem{thm}{Theorem}

\newtheorem{lem}{Lemma}
\usepackage[max2]{authblk}
\title{Kernel \textit{k}-Means, By All Means: Algorithms and Strong Consistency}
\author[1]{Debolina Paul\thanks{Joint first authors contributed equally to this work.}}
\author[2]{Saptarshi Chakraborty$^\ast$}
 \author[3]{Swagatam Das}
 \author[4]{Jason Xu\thanks{Correspondence to: jason.q.xu@duke.edu.}}
 \affil[1]{Indian Statistical Institute, Kolkata, India}
 \affil[2]{Department of Statistics, University of California, Berkeley}
 \affil[3]{Electronics and Communication Sciences Unit, Indian Statistical Institute, Kolkata, India}
 \affil[4]{Department of Statistical Science, Duke University}
\date{\vspace{-5ex}}

\begin{document}

\maketitle

\begin{abstract}
Kernel $k$-means clustering is a powerful tool for unsupervised learning of non-linearly separable data. Since the earliest attempts, researchers have noted that such algorithms often become trapped by local minima arising from non-convexity of the underlying objective function. In this paper, we generalize recent results leveraging a general family of means to combat sub-optimal local solutions to the kernel and multi-kernel settings. Called Kernel Power $k$-Means, our algorithm makes use of majorization-minimization (MM) to better solve this non-convex problem. We show the method implicitly performs annealing in kernel feature space while retaining efficient, closed-form updates, and we rigorously characterize its convergence properties both from computational and statistical points of view. In particular, we characterize the large sample behavior of the proposed method by establishing strong consistency guarantees. Its merits are thoroughly validated on a suite of simulated datasets and real data benchmarks that feature non-linear and multi-view separation.
\end{abstract}

\section{Introduction}
\label{sec:intro}
Clustering---the task of partitioning a dataset into groups based on a measure of similarity---is a cornerstone of unsupervised learning. Among a vast literature and countless applications of various clustering algorithms, the simple yet effective $k$-means method endures as the most widely used approach \citep{macqueen1967some,lloyd1982least}. A center-based method, $k$-means seeks to partition data $\{\bx_1,\dots,\bx_n\} \subset\mathbb{R}^p$ into $k$ mutually exclusive classes that minimize within-cluster variance. Denoting the  centroids $\bTheta=\{{\btheta}_1, {\btheta}_2, \dots , {\btheta}_k\}$, this can be cast as minimization of the loss function,
\begin{equation}\label{obj1}  P(\bTheta) = \sum_{i=1}^n \min_{1 \le j \le k} \|\bx_i-\btheta_j\|^2.
 \end{equation}  
Despite its successes,  $k$-means relies on assuming that data are linearly separable and even then may stop short at poor local minima due to non-convexity of \eqref{obj1}. To remedy the first issue, researchers have applied kernel methods to $k$-means \citep{6790375, dhillon2004kernel,10.1016/j.patcog.2007.05.018}, which first embed the data into a higher dimensional feature space via a nonlinear mapping. The data may become better linearly separable in this richer representation, rendering $k$-means effective. To cope with the emerging complexity of real-world data, kernel $k$-means has been subjected to several variations and analyses in recent and ongoing works \citep{NIPS2014_5236, 10.1145/2809890.2809896, 7351209, 10.5555/2946645.3007100, 8267102, 10.5555/3322706.3322718}.  
Spectral clustering offers another nonlinear approach to clustering \citep{ng2002spectral,kang2018unified,lu2018nonconvex}, acting on eigenvectors of an affinity matrix constructed from the data. An explicit relationship between kernel and spectral clustering is established by \cite{dhillon2004kernel}. 

Becoming trapped in poor local minima remains an issue, in either case, a problem that has been highlighted since the earliest uses of kernel $k$-means. In place of Lloyd's classic algorithm for solving the $k$-means problem after spectral embedding, \cite{girolami2002mercer} describes a modified EM algorithm via stochastic optimization akin to the deterministic annealing, while
\cite{dhillon2004kernel} exploit a relaxation to perform spectral initialization that is further refined via Lloyd's algorithm. A popular approach to ameliorate  $k$-means' sensitivity to initialization is based on well-chosen seedings \citep{arthur2007k,bachem2016fast}. More recent work uses annealing to temper the non-convexity by solving a sequence of better behaved problems \citep{xu2019power,chakraborty2020entropy}, reviving ideas explored by \cite{zhang1999k}. These approaches are complementary to one another and can simultaneously combat local minima. While seeding directly carries over to the kernel setting, extending the latter is nontrivial.

In this paper, we generalize power $k$-means clustering to the kernel and multi-kernel settings. The resulting algorithm, called kernel power $k$-means (KPK), admits closed-form updates while performing annealing implicitly in the embedded feature space. In contrast to many popular methods, we establish strong consistency of the centroid estimates beyond the standard convergence guarantees and derive a natural extension to multi-view learning using more than one kernel.  We show that KPK significantly improves performance in detecting linearly non-separable clusters while retaining the efficiency and simplicity of the existing methods on a suite of simulated and real data.

The paper is organized as follows: we begin with an illustrative example and overview of the necessary background. In Section \ref{sec:methods}, we formalize the proposed method and derive an efficient algorithm and its multi-kernel extension. Theoretical properties are analyzed in Section \ref{sec:theory}; in particular, we establish uniform convergence of the objective sequence and strong consistency of the estimated centroids. These merits are thoroughly validated empirically in Section \ref{sec:results}, followed by a discussion in Section \ref{sec:discussion}.

\paragraph{Motivating Example}
\begin{figure}
\begin{subfigure}{.44\textwidth}
  \centering
  \includegraphics[width=1.0\linewidth]{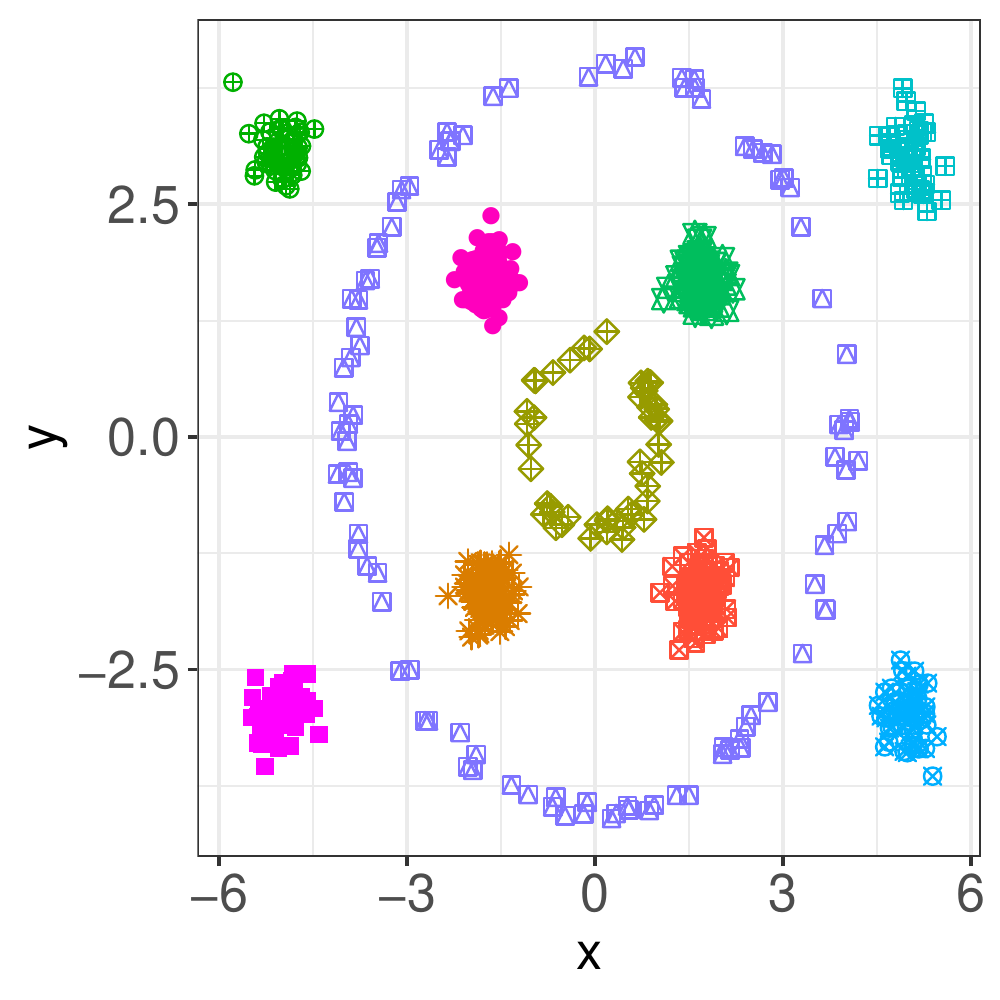}  
  \label{fig:kpk}
\end{subfigure}
\begin{subfigure}{.44\textwidth}
  \centering
  \includegraphics[width=1.0\linewidth]{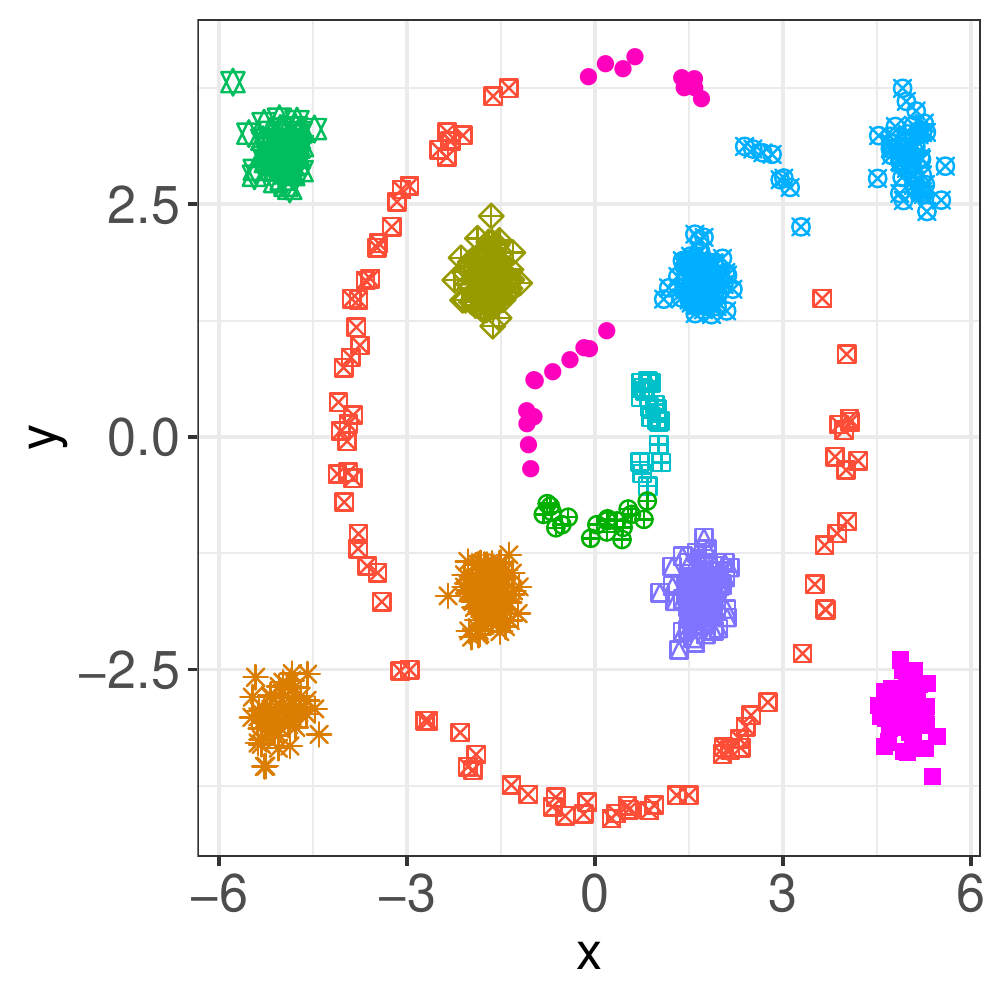}  
  \label{fig:kkmeans}
\end{subfigure} 
\caption{Despite the idealized setting, the best result using kernel $k$-means out of $20$ initializations (right) misclassifies in both rings
, with average adjusted Rand index (ARI) of $0.78$. Our proposed method (left) achieves an average ARI of $1.0$ over the $20$ restarts, indicating a perfect clustering on each trial.} 
\label{fig:motiv}
\end{figure}
Before proceeding, we motivate our contributions on a simple dataset consisting of ten clusters in two dimensions. As seen from Figure~\ref{fig:motiv}, the classes are clearly not linearly separable in the original feature space, but implicitly clustering in a higher dimension using a Gaussian kernel with $\sigma=1$ allows for successful classification. We compare kernel $k$-means to our proposed algorithm, plotting the best partitioning produced by each method out of $20$ matched initializations in Figure \ref{fig:motiv}. We see that kernel $k$-means is quite prone to falling into local minima even in this toy example, while our kernel power $k$-means method consistently arrives at the global minimum.

\subsection{Background}
\paragraph{Weighted kernel \textit{k}-means}
The (weighted) kernel version of $k$-means considers a similar objective function to \eqref{obj1} after embedding points into a new feature space by way of a non-linear mapping $\phi: \Real^p \to \Hi$ \citep{dhillon2004unified}, where $\Hi$ is a Hilbert space, replacing the distances in \eqref{obj1} by 
\[ 
\sum_{j=1}^k\sum_{\bx_i\in \mathcal{C}_j}w(\bx_i)\|\phi(\bx_i)-\btheta_j\|^2 ,
\, \text{ where } \, \,
\btheta_j=\frac{\sum_{\bb\in \mathcal{C}_j}w(\bb)\phi(\bb)}{\sum_{\bb\in \mathcal{C}_j}w(\bb)}=\argmin\limits_{\bz} \sum_{\bx_i\in \mathcal{C}_j}w(\bx_i)\|\phi(\bx_i)-\bz\|^2 .
\]
The squared Euclidean distance between $\phi(\bx_i)$ and $\theta_j$ can be expanded and given by
{\small
\begin{align*}
    &\bigg\|\phi(\bx_i)-\frac{\sum_{\bb\in \mathcal{C}_j}w(\bb)\phi(\bb)}{\sum_{\bb\in \mathcal{C}_j}w(\bb)}\bigg\|^2 =\langle \phi(\bx_i),\phi(\bx_i)\rangle+\frac{\sum_{\bb,\bc\in \mathcal{C}_j}w(\bb)w(\bc)\langle \phi(\bb),\phi(\bc)\rangle}{(\sum_{\bb\in \mathcal{C}_j}w(\bb))^2}-\frac{2\sum_{\bb\in \mathcal{C}_j}w(\bb)\langle \phi(\bx_i),\phi(\bb)\rangle}{\sum_{\bb\in \mathcal{C}_j}w(\bb)}
\end{align*}
}%
We see that all computations involving the data, enter as dot products, which can be calculated efficiently using a kernel function $\mathcal{K}(\cdot,\cdot)$. Specifically, Mercer's Theorem \citep{mercer1909} states that a continuous, symmetric, and positive semi-definite function $\mathcal{K}(\bx,\by)$ can be expressed as an inner product $\langle \phi(\bx), \phi(\by) \rangle$. Thus computing the kernel matrix $K$ on all pairs of data allows us to directly obtain and store these quantities without explicitly evaluating the mapping under $\phi$.

\paragraph{Majorization-minimization}
The MM principle has become increasingly prevalent for large-scale statistics and machine learning applications  \citep{mairal2015incremental,lange2016mm,xu2017generalized}. Instead of minimizing an objective of interest $f$ directly, an MM algorithm successively minimizes a sequence of simpler \textit{surrogate functions} $g(\btheta \mid \btheta_n)$ that need to meet the following two criteria: tangency  $g(\btheta_m \mid \btheta_m) =  f(\btheta_m)$ at the current estimate, and domination $g(\btheta \mid \btheta_m)  \geq f(\btheta)$ for all $\btheta$. The steps of an MM algorithm can then be specified by using the rule $\btheta_{m+1} := \argmin_{\btheta}\; g(\btheta \mid \btheta_m),$
which immediately implies the \textit{descent property}. Decreasing $g$ results in a descent in $f$: 
\begin{eqnarray*}
f(\btheta_{m+1}) \, \leq \, g(\btheta_{m+1} \mid \btheta_{m}) 
\, \le \,  g(\btheta_{m} \mid \btheta_{m}) 
\, = \, f(\btheta_{m}). \label{eq:descent}
\end{eqnarray*}
Note that $g(\btheta_{m+1} \mid \btheta_{m} ) \le g(\btheta_{m} \mid \btheta_{m})$ holds even when $\btheta_{m+1}$ does not minimizes $g$ exactly; instead  any step decreasing $g$ suffices. 
The MM principle provides a general strategy to transfer a complicated optimization problem onto a sequence of simpler tasks \citep{LanHunYan2000}, and incorporates the popular EM algorithm for maximum likelihood estimation with missing data as a special circumstance \citep{BecYanLan1997}. 
\paragraph{Power means}
 The power mean is a generalized mean defined by $M_s(\by)=\left(\frac{1}{k}\sum_{i=1}^k y_i^s \right)^{1/s}$ for a vector $\by$. Note $s=1$ yields the arithmetic mean, $s=-1$ the harmonic mean, and for $s>1$ it is proportional to the usual $\ell_s$-norm. 
 Power means satisfy several nice properties, including 
 the well-known power mean inequality
$M_s (\by) \le M_ t (\by)$ for any $s \le t$  \citep{steele2004cauchy}. 
Further, all power means satisfy the limits
    \begin{align}
      \lim_{s \to -\infty}M_s(\by) & =\min\{y_1,\ldots,y_k\} \label{eq:limit 1}, \\ 
    \lim_{s \to \infty}M_s(\by)&=\max\{y_1,\ldots,y_k\},  \label{eq:limit 2}
    \end{align}
but are differentiable for any finite $s$ with gradient
 $\frac{\partial}{\partial y_j}   M_ s( \by) 
 = \Big(\frac{1}{k}\sum_{i=1}^k y_i^s\Big)^{\frac{1}{s}-1} \frac{1}{k}y_j^{s-1}$ in contrast to their limiting functions.
While the $\min$ function appears in \eqref{obj1} and the harmonic mean has also been used for clustering \citep{zhang1999k,zhang2001generalized}, recent work generalizes these approaches to use the whole family of power means \citep{xu2019power,chakraborty2020entropy}. They decrease the loss 
\[ \, f_s(\Theta )=\sum_{i=1}^n M_s(\|\bx_i-\btheta_1\|^2,\ldots,\|\bx_i-\btheta_k\|^2) \, \]
iteratively along a a sequence $\{f_s\}$ where $s \rightarrow -\infty$
in place of \eqref{obj1}. By the relation \eqref{eq:limit 1}, the objectives $f_s$ approach $f_{-\infty}(\Theta)$, targeting the original $k$-means objective \eqref{obj1}, while the intermediate surfaces are better-behaved and smooth out poor local optima on the way. 
\section{Kernel Power \textit{k}-Means}\label{sec:methods}
\begin{algorithm}[ht]
\caption{Kernel Power $k$-means (KPK) Algorithm}\label{algo}
\begin{algorithmic}
\State \textbf{Input}:  $\bX \in \Real^{n \times p}$, $\eta>1$, $K$, $W^{(0)}$, $s_0=-1$. \qquad  \textbf{Output}: $\hat{W}$.
\Repeat
 \State \textbf{Step 1}: Compute $\|\phi(\bx_{i})-\btheta_j^{(m)}\|^2$ for all $i=1,\dots,n$ and $j=1,\dots,k$, using equation \eqref{update_w}.\\
 \State \textbf{Step 2}: Update $W$ by $\small \displaystyle \, \, w^{(m)}_{ij} \leftarrow \frac{\frac{1}{k}\|\phi(\bx_i)-\btheta_j^{(m)}\|^{2(s-1)}}{(\frac{1}{k}\sum_{l=1}^k\|\phi(\bx_i)-\btheta_j^{(m)}\|^{2s})^{(1-1/s)}}. \,$\\
 \State \textbf{Step 3} (Optional): Update $s \leftarrow s \cdot \eta$.\\
 \Until{objective \eqref{kpk} converges}
\end{algorithmic}
\end{algorithm}

We develop a new algorithm that performs clustering while annealing in feature space.
Let $\bx_1,\dots,\bx_n \in \Real^p$ denote the data to be clustered into $k$ disjoint clusters. Our proposed kernel power $k$-means algorithm is formulated by gradually decreasing $s$ while simultaneously seeking minimizers of the objectives
\begin{equation}\label{kpk}
    f_s(\bTheta)=\sum_{i=1}^nM_s(\|\phi(\bx_i)-\btheta_1\|^2,\dots,\|\phi(\bx_i)-\btheta_k\|^2).
\end{equation}
Here $\bTheta=\{\btheta_1,\dots, \btheta_k\}$ is the set of $k$ cluster centroids and $\phi:\Real^p \to \Hi$ is the kernel map. This formulation can be seen as a generalization that subsumes several existing methods as special cases: if $\phi(\bx)=\bx$, then \eqref{kpk} reduces  to the objective function of power $k$-means \citep{xu2019power}, while if further $s=-1$ or $s=-\infty$, we recover $k$-harmonic means \citep{zhang1999k} and the original $k$-means objective \citep{macqueen1967some}, respectively. Of course $\phi$ need not be the identity, and for a nontrivial choice of kernel, \eqref{kpk} reduces to the objective of standard kernel $k$-means upon setting $s=-\infty$.

Because the intermediate surfaces smooth out local minima, minimizing a sequence $\{f_s\}$, where $s$ decreases toward $-\infty$, will enable the benefits of annealing together with the ability to learn nonlinear separations. Utilizing a continuum of means as so has proven successful to improve $k$-means clustering, but kernelizing this idea is nontrivial due to high or infinite dimensionality of $\Hi$.

To address the resulting optimization problem, we will derive MM updates that sequentially decrease objective function \eqref{kpk}. As shown in \citep{xu2019power}, power means are concave whenever $s<1$, implying that the following tangent plane inequality holds for any anchor point $\by^{(m)}$:
\begin{equation}
\small\label{ineq}
    M_s(\by) \leq M_s(\by^{(m)}) + \sum_{j=1}^k \frac{\partial}{\partial y_j} M_s(\by^{(m)})(y_j-y^{(m)}_{j}).
\end{equation}
 Now substituting $\|\phi(\bx_i)-\theta_j\|^2$ for $y_j$ and $\|\phi(\bx_i)-\theta^{(m)}_{j}\|^2$ for $y^{(m)}_{j}$, we obtain
 \begin{align}
     &f_s(\bTheta)   \leq  f_s(\bTheta^{(m)})-\sum_{i=1}^n\sum_{j=1}^k w^{(m)}_{ij} \|\phi(\bx_i)-\btheta^{(m)}_j\|^2 +\sum_{i=1}^n\sum_{j=1}^k w^{(m)}_{ij}\|\phi(\bx_i)-\btheta_j\|^2 \, := \, \,g_s(\bTheta | \bTheta^{(m)} ) .\label{mm} 
 \end{align}
Here the partial derivatives are abbreviated \[ w^{(m)}_{ij}=\frac{\frac{1}{k}\|\phi(\bx_i)-\btheta_j^{(m)}\|^{2(s-1)}}{(\frac{1}{k}\sum_{l=1}^k\|\phi(\bx_i)-\btheta_j^{(m)}\|^{2s})^{(1-1/s)}}. \] 
The first two terms in \eqref{mm} are constant in $\bTheta$; minimizing $g_s(\bTheta | \bTheta^{(m)} )$ results in the MM iteration \[\btheta^{(m+1)}_{j}~=~\frac{\sum_{i=1}^n w^{m}_{ij} \phi(\bx_i)}{\sum_{i=1}^n w^{m}_{ij}}. \] 

These updates can be directly implemented in finite feature spaces \citep{xu2019power,chakraborty2020entropy}. However if $\Hi$ is infinite dimensional (e.g. if one takes the kernel to be Gaussian), one cannot evaluate the centroid updates in practice. 
To bypass this difficulty, we show that one can subsume an implicit update of $\btheta$ within the weight update using a kernel trick. Because $\btheta$ enters only within distance computations, we observe that
\begingroup
\allowdisplaybreaks
\begin{align}
 \|\phi(\bx_{i_0})-\btheta_j^{(m)}\|^2 
 & = \, \, \langle \phi(\bx_{i_0}), \phi(\bx_{i_0}) \rangle+ \langle \btheta_j^{(m)}, \btheta_j^{(m)} \rangle -2 \langle \phi(\bx_{i_0}),\btheta_j^{(m)} \rangle  \nonumber \\
 & =  K({i_0},{i_0})+ \bigg\langle \frac{\sum_{i=1}^n w^{(m-1)}_{ij} \phi(\bx_i)}{\sum_{i=1}^n w^{(m-1)}_{ij}}, \frac{\sum_{i=1}^n w^{(m-1)}_{ij} \phi(\bx_i)}{\sum_{i=1}^n w^{(m-1)}_{ij}} \bigg\rangle -2 \bigg\langle \phi(\bx_{i_0}),\frac{\sum_{i=1}^n w^{(m-1)}_{ij} \phi(\bx_i)}{\sum_{i=1}^n w^{(m-1)}_{ij}}\bigg\rangle  \nonumber \\
 & =  K(i_0,i_0)+ \frac{\sum_{i=1}^n \sum_{i^\prime= 1}^n w^{(m-1)}_{ij} w^{(m-1)}_{i^\prime j} \langle \phi(\bx_i), \phi(\bx_{i^\prime})  \rangle}{(\sum_{i=1}^n w^{(m-1)}_{ij})^2} -2 \frac{\sum_{i^\prime=1}^n w^{(m-1)}_{i^\prime j} \langle \phi(\bx_{i_0}), \phi(\bx_{i^\prime}) \rangle }{\sum_{i=1}^n w^{(m-1)}_{ij}}  \nonumber  \\
 & =  K(i_0,i_0)+ \frac{\sum_{i=1}^n \sum_{i^\prime= 1}^n w^{(m-1)}_{ij} w^{(m-1)}_{i^\prime j} K(i,i^\prime)}{(\sum_{i=1}^n w^{(m-1)}_{ij})^2} -2 \frac{\sum_{i^\prime=1}^n w^{(m-1)}_{i^\prime j} K(i_0, i^\prime)}{\sum_{i=1}^n w^{(m-1)}_{ij}}.\label{update_w}
\end{align}
\endgroup
Thus an MM iteration can proceed by computing pairwise distances from the data to centroids only using $K$ via \eqref{update_w}, and then updating $W$ without ever explicitly dealing with centroids $\btheta_j$. The resulting updates are simple and effective, summarized in Algorithm \ref{algo}. Note the kernel pairs $K(i,j)$ for points $\bx_i, \bx_j$ can be computed only once and cached.
 Computing and storing the kernel matrix $K$ is $\mathcal{O}(n^2)$ so that Step 1 only requires looking up $n^2 k$ necessary values, while time complexity of Step 2 is $\mathcal{O}(npk)$. Cluster labels $\bc$ are determined by reading off entries from the output $\hat{W}$: for each point $\bx_i$, note $c_i=\argmin_{1 \le j \le k} \|\phi(\bx_i)-\btheta_j\|^2$  is equivalent to assigning $c_i = \argmax_{1 \le j \le k} \hat{w}_{ij}$. 
 
 \subsection{Extension to multiple kernel settings}
For some data, transforming the features by way of a single map $\phi$ does not provide a rich enough embedding for successful clustering. In such cases a multi-view clustering method \citep{yang2018multi} may be advantageous. Using multiple kernels in the kernel $k$-means framework \citep{zhao2009multiple,du2015robust} can provide a straightforward extension to multi-view clustering. We show how to extend the above to such a multi-kernel learning approach. Here one assumes that each sample can be well-represented by a collection of feature maps $\{\phi_l(\cdot)\}_{l=1}^L$. Let kernel $\mathcal{K}_l$ correspond to the map $\phi_l(\cdot)$, and combine this collection to form a new feature map $\phi_{\balpha}(\bx):=(\sqrt{\alpha_1} \phi_1(\bx)^\top,\dots,\sqrt{\alpha_L} \phi_L(\bx)^\top)^\top$. The coefficients $\balpha=(\alpha_1,\dots,\alpha_L)$ act as weights satisfying $\sum_{l=1}^L \alpha_l=1$; 
the combined map $\phi_{\balpha}$ yields a kernel and induces a corresponding norm 
\begin{align*}
 & \mathcal{K}_{\balpha}(\bx,\by)=\langle \phi_{\balpha}(\bx) , \phi_{\balpha}(\by) \rangle = \sum_{l=1}^L \alpha_l \mathcal{K}_l(\bx,\by);\\
& \|\phi_{\balpha}(\bx)-\phi_{\balpha}(\by)\|^2= \sum_{l=1}^L \alpha_l \|\phi_l(\bx)-\phi_l(\by)\|^2_l.  
\end{align*}
Thus we can extend our approach to accommodate multiple kernels by replacing $K$ in Equation \eqref{kpk} by~$\mathcal{K}_{\balpha}$ (and hence the squared norm $\|\cdot\|^2$ by $\sum_{l=1}^L \alpha_l \|\cdot\|^2_l$). To enforce the simplex constraint on weights $\balpha$, we include \textit{entropy penalties} \citep{jing2007entropy}; this choice will be crucial toward preserving simple, closed form updates. The multi-view task can now be cast as minimization of the following objective function:
\begin{align}
    f_s(\bTheta^\prime, \bmu) & = \sum_{i=1}^n M_s(\|\phi_{\balpha}(\bx_i)-\btheta_1^\prime\|^2,\ldots, \|\phi_{\balpha}(\bx_i)-\btheta_k^\prime\|^2) + \lambda \sum_{l=1}^L \alpha_l \log \alpha_l.\label{mkpk1}
\end{align}

Here $\btheta_j^\prime= (\btheta_{j,1}^{ \top},\dots,\btheta_{j,L}^{ \top})^\top \in \Hi_1\times \dots \times \Hi_L$, and $\lambda>0$ is a regularization parameter. Note by reparametrizing  $\btheta_j=\big(\frac{1}{\sqrt{\alpha_1}}\btheta_{j,1}^{ \top},\dots,\frac{1}{\sqrt{\alpha_L}}\btheta_{j,L}^{ \top}\big)^\top,$ rewriting Equation \eqref{mkpk1} now separates over kernels $l$:
\begin{equation}\label{mkpk2}
 f_s(\bTheta, \bmu)\nonumber  = \sum_{i=1}^n M_s\big(\sum_{l=1}^L \alpha_l \|\phi_l(\bx_i)-\btheta_{1,l}\|^2,\dots, \sum_{l=1}^L \alpha_l\|\phi_l(\bx_i)-\btheta_{k,l}\|^2\big) + \lambda \sum_{l=1}^L \alpha_l \log \alpha_l.
\end{equation}

\begin{algorithm*}[ht]
\caption{Multi-Kernel Power $k$-means (MKPK) Algorithm}\label{algo2}
\begin{algorithmic}
\State \textbf{Input}: $\bX \in \Real^{n \times p}$, $\eta>1$, $K_1, \dots, K_L$, $\balpha^{(0)}$, $W^{(0)}$, $s_0=-1$. \qquad \textbf{Output}: $\hat{W}$, $\hat{\balpha}$.
\Repeat
 \State \textbf{Step 1}: Compute $\|\phi_l(\bx_{i})-\btheta_{j,l}^{(m)}\|^2$ for all $i=1,\dots,n$; $j=1,\dots,k$; $l=1,\dots,L$ via \eqref{eqq1}.\\
 \State \textbf{Step 2}: Update $W$ by $ \small \displaystyle w^{(m)}_{ij} \leftarrow \frac{\frac{1}{k} \big(\sum_{l=1}^L \alpha_l^{(m)} \|\phi_l(\bx_i)-\btheta_{j,l}^{(m)}\|\big)^{2(s-1)}}{\bigg(\frac{1}{k}\sum_{t=1}^k\big(\sum_{l=1}^L \alpha_l^{(m)}\|\phi_l(\bx_i)-\btheta_{t,l}^{(m)}\|\big)^{2s}\bigg)^{(1-1/s)}}.$
 \State \textbf{Step 3}: Update $\balpha$ by $ \small \displaystyle \alpha_l^{(m+1)} \leftarrow \frac{\exp\bigg\{- \frac{1}{\lambda}\sum_{i=1}^n\sum_{j=1}^k w_{ij}^{(m)} \|\phi_l(\bx_i)-\btheta_{j,l}^{(m)}\|^2\bigg\}}{\sum_{t=1}^L\exp\bigg\{- \frac{1}{\lambda}\sum_{i=1}^n\sum_{j=1}^k w_{ij}^{(m)} \|\phi_l(\bx_i)-\btheta_{j,t}^{(m)}\|^2\bigg\}}.$
 \State \textbf{Setp 4} (Optional): Update $s \leftarrow s \cdot \eta$.\\
 \Until{objective \eqref{mkpk1} converges}
\end{algorithmic}
\end{algorithm*} 

\paragraph{Optimization}
The entropy incentive appearing as the final term in \eqref{mkpk2} will enable us to retain efficient, closed form MM steps. In Equation \eqref{ineq}, we substitute $\sum_{l=1}^L \alpha_l \|\phi_l(\bx_i)-\btheta_{j,l}\|^2$ for $y_j$ and $\sum_{l=1}^L \alpha_l^{(m)} \|\phi_l(\bx_i)-\btheta_{j,l}^{(m)}\|^2$ for $y_j ^{(m)}$;  via similar arguments to the single kernel case detailed in the appendix, we derive an MM algorithm summarized in Algorithm \ref{algo2}. 
As before, 
since $\btheta_{j,l}$ may lie in an infinite dimensional Hilbert space, denoting $K_l(i,j)=\mathcal{K}_l(\bx_i,\bx_j)$, we employ a kernel trick to compute differences
\begin{align}\label{eqq1}
     &\|\phi_l(\bx_{i_0})-\btheta_{j,l}^{(m)}\|^2 
     = K_l(i_0,i_0) + \frac{\sum\limits_{i,i^\prime=1}^n  w^{(m-1)}_{ij} w^{(m-1)}_{i^\prime j} K_l(i,i^\prime)}{(\sum_{i=1}^n w^{(m-1)}_{ij})^2} 
    -2\frac{\sum\limits_{i^\prime=1}^n w^{(m-1)}_{i^\prime j} K_l(i_0, i^\prime)}{\sum_{i=1}^n w^{(m-1)}_{ij}}.
\end{align}

\section{Theoretical properties}\label{sec:theory}

We now derive several properties related to the convergence of our method. The first two results characterize the sequence of minimizers, and extend arguments from \cite{xu2019power} to the kernel setting, with proofs in the appendix. We then develop our main theoretical result establishing strong consistency of the centroids.
The results are distribution-free in that we only assume that the data $\bX_1,\dots,\bX_n \in \Real^p$ are independently and identically distributed according to some distribution $P$ with compact support $C \subset \Real^p$. The exposition focuses on a single kernel; these arguments also apply to the multi-kernel case, and the results are formally extended in the appendix for completeness.

We first show that the minima $\bTheta_{n,s}$ of the surrogates lie in a convex hull in the image $\phi(C)$. For notational simplicity, let $\mathscr{C}(A)$ denote the closed convex hull of a set $A$.
\begin{thm}\label{hull}
Assume $\phi : C \longrightarrow \mathcal{H}$ to be a function from C to some Hilbert space $\mathcal{H}$. Let $\bTheta_{n,s}$ be the minimizer of $f_s(\bTheta)$, $s \leq 1$ . Then $\bTheta_{n,s}$ lies in the compact Cartesian product $\mathscr{C}(\phi(C))^k$. 
\end{thm}
The next result strengthens Equation \eqref{eq:limit 2}: surrogates converge \textit{uniformly} on the compact set $\phi(C)^k$. 
\begin{thm}
\label{uniform}
For any decreasing sequence $\{s_m\}_{m=1}^\infty$ such that $s_1 \leq 1$ and $s_m \to -\infty$, the functions $f_{s_m}(\bTheta)$ converge uniformly to $f_{-\infty}(\bTheta)$ on $\mathscr{C}(\phi(C))^k$.
\end{thm} 
In particular, the uniform convergence in Theorem \ref{uniform} immediately implies that the sequence of minimizers $\bTheta_{n,s}$ converges to $\bTheta_{n,\infty}$, the minimizer of the kernel \textit{k-}means objective $f_{-\infty}(\bTheta)$.

Toward proving strong consistency, we move to establish a Uniform Strong Law of Large Numbers (USLLN) which plays a pivotal role in the proof of our main theorem.
To lighten notation, abbreviate \[ \M_s(\bx,\bTheta)=M_s(\|\phi(\bx)-\btheta_1\|^2,\dots,\|\phi(\bx)-\btheta_k\|^2), \] and define $\bTheta^\ast$ to be the set of $k$ centroids minimizing the population-level loss \[ \Psi(\bTheta,P)=\int \min_{1 \le j \le k} \|\phi(\bx)-\btheta_j\|^2 dP.\] This mirrors the notation that $\bTheta_{n,s}$ are the minimizers of $\int \M_s(\bx,\bTheta)dP_n$. Establishing consistency amounts to showing that $\bTheta_{n,s} \overset{a.s.}{\to} \bTheta^\ast$ \, as $n \to \infty$ and $s \to -\infty $; we do so under the  regularity conditions:  
\begin{enumerate}
    \item[A1.] The map $\phi : (C,\|\cdot\|_2) \to (\Hi,\|\cdot\|)$ is continuous.
    \item[A2.] For any $r>0$, there exists $\epsilon>0$ such that for all $\bTheta \in \mathscr{C}(\phi(C))^k \setminus B(\bTheta^\ast,r)$, we have $\Psi(\bTheta,P)> \Psi (\bTheta^\ast,P)+\epsilon$.
\end{enumerate}

Before proving the results, we remark that A1 and A2 are quite mild assumptions: $\phi$ is only assumed continuous and need not be Lipschitz as in \cite{yan2016robustness}. Commonly used choices such as the Gaussian and polynomial kernels all satisfy this assumption. A2 is also standard \citep{pollard1981strong,chakraborty2019strong} and  only posits that the population minimizer of $\Psi(\cdot,P)$ is identifiable.
\begin{lem}\label{uslln}
    (USLLN) Under  A1 and A2, for $s_0\le-1$,  $$\sup_{s \le s_0, \bTheta \in \mathscr{C}(\phi(C))^k\}}\left|\int \M_s(\bx,\bTheta) dP_n-\int \M_s(\bx,\bTheta) dP\right| \to 0$$ almost surely under $P$.
\end{lem} 
\begin{proof}
    Define  $\mathcal{G}=\{\M_s(\bx,\bTheta): s~\le~s_0 \text{ and }\bTheta \in \mathscr{C}(\phi(C))^k\}$. It is enough to show that for any $\epsilon>0$, there exists $\mathcal{G}_\epsilon \subset \mathcal{G}$ such that $|\G_\epsilon|<\infty$ and for all $g \in \mathcal{G}$, there exist $\dot{g},\bar{g} \in \G_\epsilon$ with $\dot{g} \leq g \leq \bar{g}$ such that $\int(\bar{g}-\dot{g})dP < \epsilon$.
    
    We begin by observing that since $\M_s(\bx,\bTheta)$ converges uniformly to $\min_{\btheta \in \bTheta}\|\phi(\bx)-\btheta\|^2$, as $s \to -\infty$ (due to Theorem \ref{uniform}), we can find $s_1 < s_0$ such that if $s \le  s_1$, then \[ |\M_s(\bx,\bTheta) - \min_{\btheta \in \bTheta}\|\phi(\bx)-\btheta\|^2| < \epsilon/8 \] for all $\bTheta \in \mathscr{C}(\phi(C))^k$. Thus, for all $s,s^\prime \le  s_1$,
    \begin{equation}\label{lab1}
        \left|\M_s(\bx,\bTheta) - \M_{s^\prime}(\bx,\bTheta)\right| < \epsilon/4.  
    \end{equation}
    We begin by noting that $\phi(C)$ is the image of a compact set $C$ under a continuous map $\phi$, and is therefore itself compact in the metric space $(\mathcal{H},\|\cdot\|)$. Since $\Hi$ is locally convex and completely metrizable, $\mathscr{C}(\phi(C))$ is compact (Theorem 5.35 of \cite{aliprantis1986border}). Since $\mathcal{M}_s(\bx,\bTheta)$, as a function of $(s,\bx,\bTheta)$ is  continuous on the compact set $[s_1,s_0] \times C \times \mathscr{C}(\phi(C))^k$, it is uniformly continuous by the Heine-Cantor theorem \citep{apostol1964mathematical}. 
    This implies that for any $\epsilon > 0$, we can choose $\delta$ small enough such that for any two sets of centroids $\bTheta, \bTheta'$ such that $\|\btheta_j-\btheta_j^\prime\|< \delta$ for all $j=1,\dots,k$ and $|s-s^\prime|< \delta$ (i.e. $s,s^\prime \in [s_1,s_0]$), we have 
    \begin{align}
        &\left| M_s(\|\phi(\bx)-\btheta_1\|^2,\dots,\|\phi(\bx)-\btheta_k\|^2) - M_{s^\prime}(\|\phi(\bx)-\btheta_1^\prime\|^2,\dots,\|\phi(\bx)-\btheta^\prime_k\|^2) \right| < \epsilon/4. \label{eq del}    
    \end{align}
    
    We write $\bTheta^\prime = \{\btheta^\prime_1,\dots,\btheta^\prime_k\}$, and now note that 
    \begin{align}
        | M_s(\bx,\bTheta)-M_{s_1}(\bx,\bTheta^\prime)|
        \,\le\, & | M_s(\bx,\bTheta)-M_{s_1}(\bx,\bTheta)| + | M_{s_1}(\bx,\bTheta^\prime)-M_{s_1}(\bx,\bTheta^\prime)| \nonumber \\
         \,\le\, & \epsilon/4 + \epsilon/4 = \epsilon/2 \label{eq del2}
    \end{align}
    The last inequality follows from \eqref{lab1} and \eqref{eq del}.
    Compactness further implies that $[s_1,s_0]$ and $\mathscr{C}(\phi(C))$ are totally bounded, so we may create two $\delta$-nets $N_\delta^{(1)}$ and $N_\delta^{(2)}$ of $[s_1,s_0]$ and  $\mathscr{C}(\phi(C))$, respectively. That is, $|N_\delta^{(1)}|,|N_\delta^{(2)}|< \infty$, and for all $s \in [s_1,s_0]$ and $\btheta \in \mathscr{C}(\phi(C))$, there exists $s^\prime \in N_\delta^{(1)}$ and $\btheta^\prime \in N_\delta^{(2)}$ such that $\|\btheta-\btheta^\prime\|<\delta$ and $|s-s^\prime| < \delta$.
    Now, choose \[ \G_\epsilon=\bigg\{\max\{M_{s^\prime}(\|\phi(\bx)-\btheta_1\|^2,\dots,\|\phi(\bx)-\btheta_k\|^2)\pm \epsilon/2,0\}: \btheta_1,\dots,\btheta_k \in N_\delta^{(2)},\, s \in N_\delta^{(1)}\cup\{s_1\}\bigg\}.\]
    For any $g \in \mathcal{G}$ and $\bTheta \in \mathscr{C}(\phi(C))^k$, if $s \in [s_1,s_0]$, 
    we may take 
    \begin{equation}\small
    \label{eq:otherwise}
        \dot{g}_{\bTheta}(\bx)=\left(\mathcal{M}_s(\bx,\bTheta^\prime)-\frac{\epsilon}{2}\right)_+; \,\bar{g}_{\bTheta}(\bx) =\mathcal{M}_s(\bx,\bTheta^\prime)+\frac{\epsilon}{2}
    \end{equation}
Otherwise, if $s<s_1$, we replace $M_s'$ by $M_{s_1}$ in the definitions of $\dot{g}_{\bTheta}(x), \, \bar{g}_{\bTheta}(x), $ in \eqref{eq:otherwise}.
    Here $\btheta_j^\prime \in N_\delta$ and $\|\btheta_j-\btheta_j^\prime\| < \delta$ for all $j=1,\dots,k$, and by the construction of $\bTheta^\prime=\{\btheta_1^\prime,\dots,\btheta^\prime_k\}$ in \eqref{eq del} and \eqref{eq del2}, it follows that $\dot{g} \leq g \leq \bar{g}$. \,
    It remains to show that $\int(\bar{g}-\dot{g})dP < \epsilon$. To see this, 
    \[ \int(\bar{g}-\dot{g})dP \,  = \,  \int \big( M_s(\bx,\bTheta^\prime)+\frac{\epsilon}{2}-\max\big\{M_s(\bx,\bTheta^\prime)-\frac{\epsilon}{2},0\big\}\big) dP   \le \quad  \epsilon \int dP \,\, = \,\, \epsilon.\]
\end{proof}
\begin{thm}\label{main theorem}
    (Strong Consistency) Under  A1 and A2, $\bTheta_{n,s} \overset{a.s.}{\to} \bTheta^\ast$ as $n \to \infty$ and $s \to -\infty$.
\end{thm}
 \begin{proof}
     We must show for arbitrarily small $r>0$, that the minimizer $\bTheta_{n,s}$ eventually lies inside the ball $B(\bTheta^\ast,r)$. From A2, it suffices to show that for all $\eta>0$, there exists $N_1>0$ and $N_2<0$ such that $n>N_1$ and $s<N_2$ implies that $\Psi(\bTheta_{n,s},P) - \Psi(\bTheta^\ast,P) \leq \epsilon$ 
     almost everywhere $[P]$.  
      We observe that \[\Psi(\bTheta_{n,s},P)-\Psi(\bTheta^\ast,P) = \xi_1 + \xi_2 + \xi_3, \quad \text{where} \] 
      \[ \xi_1  = \Psi(\bTheta_{n,s},P) - \int \M_s(\bx,\bTheta_{n,s})dP;\] \[\xi_2  = \int \M_s(\bx,\bTheta_{n,s})dP - \int \M_s(\bx,\bTheta_{n,s})dP_n; \] 
      \[\xi_3  = \int \M_s(\bx,\bTheta_{n,s})dP_n - \Psi(\bTheta^\ast,P).\] 
    We first choose $N_2<0$ such that if $s \le N_2$, then  \[\left|\min_{\btheta \in \bTheta} \|\phi(\bx)-\btheta\|^2 - \M_s(\bx,\bTheta)\right| < \epsilon/6 \] for all $\bx \in C$ and $\bTheta \in \phi(C)^k$. This implies that 
    \begin{align*}
        \xi_1 & = \Psi(\bTheta_{n,s},P) - \int \M_s(\bx,\bTheta_{n,s})dP \\
        & =\int \bigg(\min_{\btheta \in \bTheta} \|\phi(\bx)-\btheta\|^2 - \M_s(\bx,\bTheta_{n,s}) \bigg) dP\\
        & \leq \frac{\epsilon}{6} \int dP = \frac{\epsilon}{6}.
    \end{align*}
    Appealing to Lemma \ref{uslln}, we can choose $N_1$ large enough such that $n> N_1$ implies that $\xi_2< \epsilon/3$. To bound the third term $\xi_3$, we observe the following:
    \begingroup
    \allowdisplaybreaks
    \begin{align}
    \xi_3 & = \int \M_s(\bx,\bTheta_{n,s})dP_n - \Psi(\bTheta^\ast)  \nonumber\\
    &\le \int \M_s(\bx,\bTheta^\ast)dP_n - \Psi(\bTheta^\ast) \label{m1} \\ 
    &  \le \int \M_s(\bx,\bTheta^\ast)dP - \Psi(\bTheta^\ast) + \epsilon/6 \label{m2}\\ 
    &  \le \int \{\min_{\btheta \in \bTheta^\ast} \|\phi(\bx) - \btheta\|^2 + \epsilon/6 \} dP - \int \min_{\btheta \in \bTheta^\ast} \|\phi(\bx) - \btheta\|^2 dP +\epsilon/6 \, = \,  \epsilon/3 \label{m3} 
    \end{align}
    \endgroup
    Eq. \eqref{m1} holds since $\bTheta_{n,s}$ is the minimizer for $\int \M_s(\bx, \bTheta)dP$, and Eqs. \eqref{m2} and \eqref{m3} follow from Lemma \ref{uslln} and Theorem \ref{uniform}. Thus,  \[\Psi(\bTheta_{n,s},P)-\Psi(\bTheta^\ast,P) = \xi_1 + \xi_2 + \xi_3 \le \epsilon/6 + \epsilon / 3+ \epsilon/3< \epsilon. \] 
    \end{proof}

\section{Results and performance}\label{sec:results}
\begin{figure*}[ht!]
    \centering \hspace{-3pt}
    \includegraphics[height=0.275\linewidth,width=.345\linewidth]{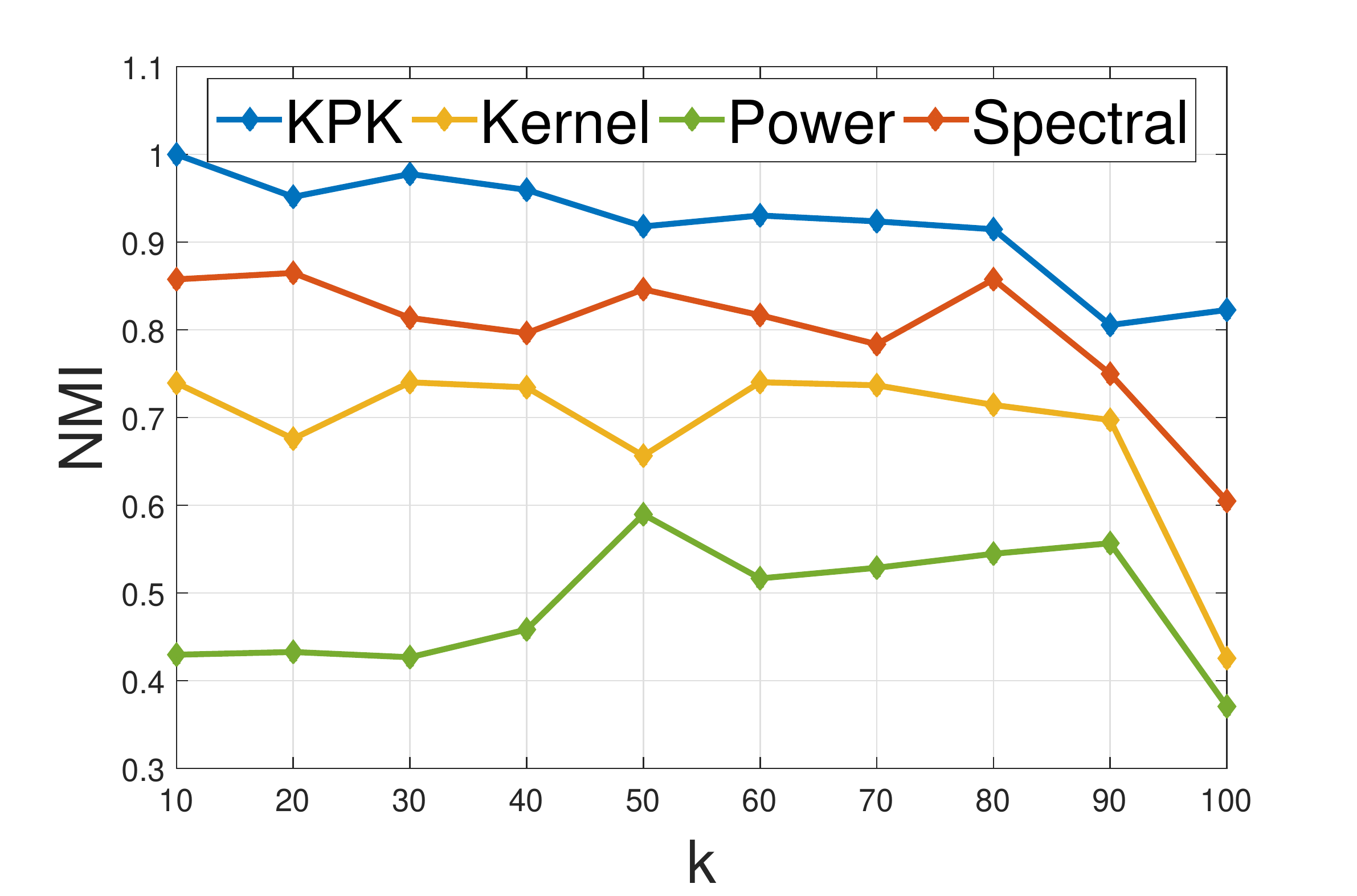}
    \hspace{-4pt}    \includegraphics[height=0.27\linewidth,width=.315\linewidth]{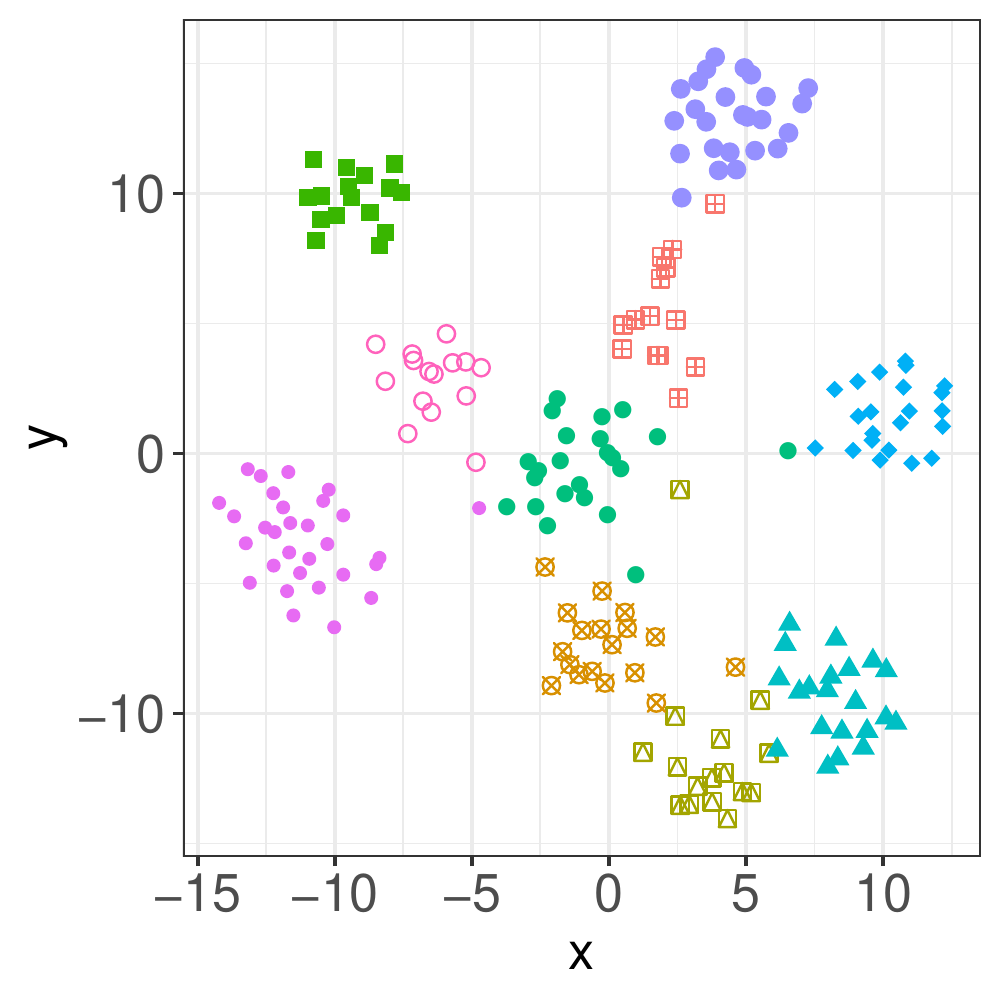} \hspace{-1pt}
    \includegraphics[height=0.27\linewidth,width=.315\linewidth]{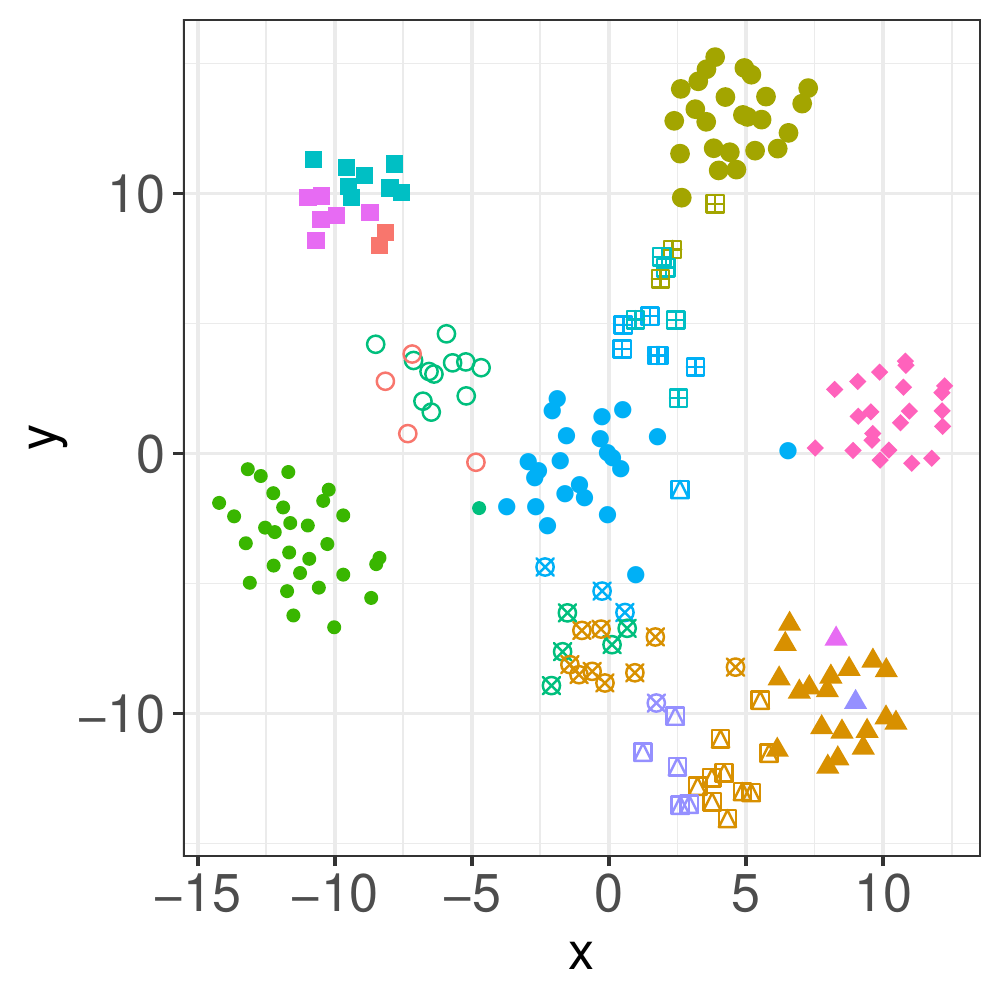}    \hspace{-1pt}
    \caption{Performance of peer algorithms in terms of average NMI values as $k$ varies from $10$ to $100$, and $t$-SNE plots color-coded with the partitions produced by KPK (middle) and kernel $k$-means (right) with $k=10$. }
    \label{fig_nmi}
\end{figure*}
\paragraph{Simulation study} While the motivating example in Section \ref{sec:intro} already shows that the proposed method successfully evades local minima on a classic simulation setup, we turn to a closer empirical analysis in a more difficult non-linear setting.
We draw $k$ true cluster centroids $\bmu_j$ uniformly along the surface of a unit sphere $S_{19}=\{\bx:\|\bx\|_2=1\} \subset \Real^{20}$, and assign ground truth labels $c_i$ to each observation $\bx_i$ uniformly. A point assigned to cluster $j$ is drawn from the von-Mises-Fisher distribution \citep{downs1972orientation} normalized to lie on $S_{19}$ with mean direction $\bmu_j$ and $\kappa=30$: that is,
\begin{align*}
 &   \bmu_1,\dots,\bmu_k \sim \text{Unif}(S_{19}); \quad c_i \sim \text{Unif}\{1, \ldots, k \}; \\
&     \bx_i|c_i \sim VMF(\bmu_{c_i},\kappa).
\end{align*}
We consider performance while varying the true number of clusters $k$ between $10$ and $100$, which increases the number of local optima. For each $k$, we generate $20$ datasets with $n=20$ observations each, and run the competing methods from $20$ matched initializations  until convergence. We use the Gaussian kernel $\mathcal{K}(\bx,\by)=\exp\{-\cos^{-1}(\bx^\top \by)/(2\sigma^2)\}$, with bandwidth parameter $\sigma=1$.


Our comparison will focus between the proposed method and kernel $k$-means (under same choice of kernel) \citep{girolami2002mercer}, power $k$-Means \citep{xu2019power} and spectral clustering \citep{ng2002spectral}. There are many variations that entail higher complexity and additional hyperparameters, while our method can be seen as a drop-in improvement of kernel $k$-means.  
We use Normalized Mutual Information (NMI) \citep{vinh2010information} to assess the partitioning obtained by each approach to the ground truth labels, whose value ranges between $1$ indicating perfect recovery and $0$.
Average NMI values are shown in Figure \ref{fig_nmi}; while all methods struggle as $k$ increases, it is clear that kernel power $k$-means outperforms peer algorithms while retaining their simplicity. A very similar trend can be observed for the average ARI values and is reported in the appendix. A $t$-SNE \citep{maaten2008visualizing} visualization of a simulated data with $k=10$, color-coded with the clusterings obtained by both KPK and kernel $k$-means are shown  in Figure \ref{fig_nmi}.    
\begin{table*}[ht]
     \caption{Average NMI values and average rank on real data; $+$ ($\approx$) indicates statistically significant (equivalent) result with respect to the best performing algorithm for that row.}
    \label{tab single}
    \centering
    \begin{tabular}{c c c c c}
    \hline
         Dataset & Kernel Power $k$-means & Kernel $k$-means & Power $k$-means & Spectral Clustering \\
         \hline
         Yale & \textbf{0.5921}(1) & 0.5199$^+$(2) & 0.1714$^+$(4) & 0.5241$^+$(3)\\
        JAFFE & \textbf{0.9278}(1) & 0.8501$^+$(4) & 0.8974$^\approx$(2) & 0.8752$^\approx$(3)\\
        TOX171 & 0.3328$^\approx$(2) &  0.1984$^+$(3) &  0.1760$^+$(4) &  \textbf{0.3552}(1)\\
        Seeds & \textbf{0.7502}(1) & 0.7247$^\approx$(3) & 0.7384$^\approx$(2) & 0.7239$^\approx$(4)\\
        Lung & \textbf{0.6539}(1)  &  0.5728$^+$(2) &  0.1945$^+$(4) &  0.5255$^+$(3) \\
        Isolet & \textbf{0.8466}(1) & 0.7694$^+$(3) & 0.7582$^+$(4) & 0.7882$^+$(2)\\
        Lung discrete & \textbf{0.8261}(1) & 0.5320$^+$(4) & 0.6967$^+$(3) & 0.7340$^+$(2)\\
        COIL20 & \textbf{0.8082}(1) & 0.6882$^+$(4) & 0.7698$^\approx$(2) & 0.7083$^+$(3)\\
        GLIOMA & \textbf{0.6297}(1) & 0.4085$^+$(3) &  0.5931$^\approx$(2) & 0.2509$^+$(4) \\
         \hline
         Average Rank & \textbf{1.11} & 3.11 & 3 &  2.78\\
         \hline
    \end{tabular}
\end{table*}
\begin{table*}[ht]
            \caption{Average NMI values and (ranks) comparing single- and multi-kernel methods.}
        \label{tab_multi}
        \centering
    \begin{tabular}{c c c c c c c}
   \hline
      Datasets & Kernel $k$-means & Spectral & MKKM & RKKM & RMKKM & MKPK\\
      \hline
       Yale  & 0.4207 (6) & 0.4479 (4) &  0.5006 (3) &  0.4287 (5)  & \textbf{0.5558} (1) & 0.5482 (2)\\
       Jaffe & 0.7148 (5) & 0.5935 (6) & 0.7979 (3) & 0.7401 (4)  & 0.8937 (2) & \textbf{0.9247} (1)\\
       ORL &  0.6336 (6) &  0.6674 (4) & 0.6886 (3) & 0.6391 (5)  & 0.7483 (2) &\textbf{ 0.7876} (1)\\
       COIL20 & 0.6357 (5) & 0.5434 (6) & 0.7064 (3) & 0.6370 (4) &  0.7734 (2) & \textbf{0.7763} (1)\\
    \hline 
    Average Rank & 6.5 & 5 & 3 & 4.5 & 1.75 & \textbf{1.25} \\
    \hline
    \end{tabular}
\end{table*}\paragraph{Real data analysis}
We begin by studying performance on classic clustering benchmark datasets. The datasets JAFFE and Seeds are collected from \cite{670949} and UCI machine learning repository \citep{Dua:2019} respectively. The rest are collected from the ASU feature selection repository\footnote{\url{http://featureselection.asu.edu/datasets.php}} \citep{li2018feature}. 
The average NMI values obtained for all the peer algorithms using a single kernel are summarized in Table \ref{tab single}.

All data are centered and scaled before the experiment. The Gaussian kernel is used in all the experiments. The bandwidth parameter $\sigma$ for the Gaussian kernel is chosen as $ \sqrt{\sum_{i=1}^n \sum_{j=1}^n \|\bx_i-\bx_j\|_2^2)/n(n-1)}$ following \cite{calandriello2018statistical,wang2019scalable}. On each dataset, the same kernel is used across all methods where applicable. The value of $s_0$ and $\eta$ in Power $k$-Means and Kernel Power $k$-Means is taken to be $-1$ and $1.04$ respectively, updating $s$ every $5$ iterations. All algorithms are iterated until convergence and repeated over $20$ matched random initializations; an analogous study with $k$-means++ initialization appears in the appendix. We report mean performance and assess the statistical significance of observed differences via Wilcoxon's signed-rank test \citep{wasserman2006all}. Results appear in Table \ref{tab single}, where  $(+)$ 
indicates the difference from the best performer on a given dataset is significant at the 5\% level.  It is clear from the average NMI values that kernel power $k$-means outperforms peer algorithms in almost all cases, often with statistical significance.
\paragraph{Multi-view data} We next examine data that have been considered in past multi-view clustering studies \citep{du2015robust}. We employ $12$ different kernels for our experiments on multi-view datasets as in \cite{du2015robust}: $7$ Gaussian kernels, $4$ polynomial kernels, and one cosine kernel, choosing the parameters of the kernel functions following  \cite{du2015robust}. We normalize all the kernels by $\mathcal{K}(\bx_i,\bx_j)=\mathcal{K}(\bx_i,\bx_j)/\sqrt{\mathcal{K}(\bx_i,\bx_i)\mathcal{K}(\bx_j,\bx_j)}$ and rescale in $[0,1]$. 
We compare our proposed Multi-Kernel Power $k$-Means (MKPK) in the same multi-kernel setup with Kernel $k$-Means, Spectral Clustering \citep{ng2002spectral}, Multiple Kernel $k$-Means (MKKM) \citep{huang2011multiple}, Robust Kernel $k$-Means (RKKM) \citep{du2015robust} and Robust Multiple Kernel $k$-Means (RMKKM) \cite{du2015robust}. The average NMI values obtained for $20$ repetitions for each of the peer algorithms (the results for MKKM, RKKM, and RMKKM are quoted from \cite{du2015robust}) are summarized in Table \ref{tab_multi}, which showcases the promise of MKPK. 

It should be noted that our method outperforms competing methods despite maintaining a simpler update scheme and computational complexity of only
$\mathcal{O}(n^2kL)$ per iteration, much lower than the $\mathcal{O}((n^3+n^2+n)L+(n^2+n)k$ cost of RMKKM and comparable to the $\mathcal{O}(n^2kL)$ complexity of MKKM. As a drop-in replacement for kernel $k$-means, this cost can be further reduced using existing acceleration methods for the computation of $\mathcal{K}$.

\section{Discussion}\label{sec:discussion}
This paper utilizes the continuum of power means to define and solve well-behaved optimization problems that approach the original kernel $k$-means objective. We show that kernel power $k$-means elegantly brings this annealing scheme to bear via MM, bridging recent developments that successfully combat local minima in the original feature space to non-linear classification tasks. We extend existing theoretical results and additionally derive novel large-sample properties of our method for kernel and multi-kernel setups. We emphasize the simplicity and low complexity of our approach; it can be seen as a drop-in replacement for improving kernel $k$-means. Our empirical studies show that it consistently outperforms standard kernel $k$-means and comparable variants.

Several directions remain open. A thorough theoretical investigation of annealing rates is lacking, and characterizing optimal schedules to decrease $s$ toward $-\infty$ is both of methodological interest and practical relevance. Second, though power means objectives are non-linear and do not directly yield equivalent trace problem formulations, future work may explore this direction to seek explicit connections between the proposed method and approaches such as spectral clustering that have been established for standard kernel $k$-means. The broader idea that smoothing out local minima can significantly improve performance is not tied to the choice of Euclidean distance in the loss function. Thus, extensions to other classes of divergences are warranted and remain fruitful avenues for future work.


\clearpage
\appendix
\begin{center} {\LARGE \bf Appendix }
\end{center}
\section{Proofs from Section 3}
\subsection{Theorem 1}
\begin{proof}
Let $P_{\mathscr{C}(\phi(C))} (\btheta)$ denote the projection of $\btheta$ onto $\mathscr{C}(\phi(C))$. Now for any $\bv \in \mathscr{C}(\phi(C))$, we use the obtuse angle condition to obtain, $\langle \btheta-P_{\mathscr{C}(\phi(C))} (\btheta), \bv-P_{\mathscr{C}(\phi(C))} (\btheta) \rangle \leq 0$. Since $\bx_i \in \phi(C)$, we obtain,
\begin{align*}
    \|\bx_i-\btheta_j\|^2 & = \|\bx_i-P_{\phi(C)} (\btheta_j)\|^2 + \|P_{\mathscr{C}(\phi(C))} (\btheta_j)-\btheta_j\|^2 -2 \langle \btheta-P_{\mathscr{C}(\phi(C))} (\btheta_j), \bx_i-P_{\mathscr{C}(\phi(C))} (\btheta_j) \rangle\\
    & \geq \|\bx_i-P_{\mathscr{C}(\phi(C))} (\btheta_j)\|^2 + \|P_{\mathscr{C}(\phi(C))} (\btheta_j)-\btheta_j\|^2.
\end{align*}
Now since, $M_s(\cdot)$ is an increasing function in each of its argument, if we replace $\btheta_j$ by $P_{\mathscr{C}(\phi(C))} (\btheta_j)$ in $M_s(\|\bx_i-\btheta_1\|^2,\dots,\|\bx_i-\btheta_k\|^2)$, the objective function value doesn't go up. Thus we can effectively restrict our attention to $\mathscr{C}(\phi(C))^k$. Now since the function $f_s(\cdot)$ is continuous on the compact set $\mathscr{C}(\phi(C))^k $, it attains its minimum on $\mathscr{C}(\phi(C))^k$. Thus, $\bTheta^* \in \mathscr{C}(\phi(C))^k$. 
\end{proof}

\subsection{Theorem 2}
\begin{proof}
For any $\bTheta \in \phi(C)^k$, $f_{s_m}(\bTheta)$ decreases monotonically to $f_{-\infty}(\bTheta)$ (this is due to the power mean inequality). Since $\phi(C)^k$ is compact, the result follows immediately upon applying Dini's theorem from real analysis \cite{apostol1964mathematical}.
\end{proof}

\section{MM for Multi-kernel Setting}
The majorization is supplied by the following:
\begin{align*}
f_s(\bTheta,\balpha)  \,\leq\, & f_s(\bTheta^{(m)},\balpha^{(m)})-\lambda \sum_{l=1}^L \alpha_l^{(m)} \log \alpha_l^{(m)} -\sum_{i=1}^n\sum_{j=1}^k w^{(m)}_{ij} \sum_{l=1}^L \alpha_l^{(m)} \|\phi_l(\bx_i)-\btheta_{j,l}^{(m)}\|^2\\
&+\sum_{i=1}^n\sum_{j=1}^k w^{(m)}_{ij}\sum_{l=1}^L \alpha_l \|\phi_l(\bx_i)-\btheta_{j,l}\|^2 + \lambda \sum_{l=1}^L \alpha_l \log \alpha_l , \qquad \text{where}
\end{align*}
\[w^{(m)}_{ij}=\frac{1}{k} \big(\sum_{l=1}^L \alpha_l^{(m)} \|\phi_l(\bx_i)-\btheta_{j,l}^{(m)}\|\big)^{2(s-1)} \bigg(\frac{1}{k}\sum_{t=1}^k\big(\sum_{l=1}^L \alpha_l^{(m)}\|\phi_l(\bx_i)-\btheta_{t,l}^{(m)}\|\big)^{2s}\bigg)^{(1/s-1)} . \]

\section{Theorem Generalizations to Multi-kernel Setting}
\begin{thm}\label{hull 2}
Assume $\phi : C \longrightarrow \mathcal{H}$ to be a function from C to some Hilbert space $\mathcal{H}$. Let $\bTheta_{n,s}$ be the minimizer of $f_s(\bTheta)$, $s \leq 1$ . Then $\bTheta_{n,s}$ lies in the compact Cartesian product $(\mathscr{C}(\phi_1(C)) \times \dots \times \mathscr{C}(\phi_L(C)))^k$.
\end{thm}
\begin{proof}
    Let $P_l(\btheta)$ be the projection of $\btheta$ onto $\phi_l(C)$, $\forall l=1,\dots,L$. Observe that for any $\bv \in \mathscr{C}(\phi(C))$, $\langle \btheta - P_l(\btheta), \bv- P_l(\btheta)  \rangle \le 0$. We observe the following:
    \begin{align*}
        \sum_{l=1}^L \alpha_l \|\bx_i - \btheta_{j,l}\|^2 & = \sum_{l=1}^L  \alpha_l \bigg[ \|\bx_i - P_l(\btheta_{j,l})\|^2 + \|P_l(\btheta_{j,l}) - \btheta_{j,l}\|^2 -2 \langle \btheta_{j,l} - P_l(\btheta_{j,l}), \bx_i- P_l(\btheta_{j,l})  \rangle \bigg]\\
        & \ge \sum_{l=1}^L  \alpha_l \bigg[ \|\bx_i - P_l(\btheta_{j,l})\|^2 + \|P_l(\btheta_{j,l})- \btheta_{j,l}\|^2\bigg].
    \end{align*}
    Now, since $M_s(\cdot)$ is an increasing function in each of its arguments, replacing $\btheta_{j,l}$ by $P_l(\btheta_{j,l})$ in $M_s( \sum_{l=1}^L \alpha_l \|\bx_i - \btheta_{1,l}\|^2 ,\dots, \sum_{l=1}^L \alpha_l \|\bx_i - \btheta_{k,l}\|^2 )$ does not increase the objective functional value. Thus, we can effectively restrict our search for $(\bTheta_{n,s}, \balpha)$ to the compact set $(\mathscr{C}(\phi_1(C)) \times \dots \times \mathscr{C}(\phi_L(C)))^k \times [0,1]^L$. Since $f_s(\cdot,\cdot)$ is a continuous function on the compact set, $(\mathscr{C}(\phi_1(C)) \times \dots \times \mathscr{C}(\phi_L(C)))^k \times [0,1]^L$, it attains its minima within that compact set. Thus, $\bTheta_{n,s} \in (\mathscr{C}(\phi_1(C)) \times \dots \times \mathscr{C}(\phi_L(C)))^k$. 
\end{proof}

\begin{thm}
\label{uniform}
For any decreasing sequence $\{s_m\}_{m=1}^\infty$ such that $s_1 \leq 1$ and $s_m \to -\infty$, the functions $f_{s_m}(\bTheta,\balpha)$ converge uniformly to $f_{-\infty}(\bTheta,\balpha)$ on $(\mathscr{C}(\phi_1(C)) \times \dots \times \mathscr{C}(\phi_L(C)))^k$.
\end{thm}
\begin{proof}
for any $\bTheta \in (\mathscr{C}(\phi_1(C)) \times \dots \times \mathscr{C}(\phi_L(C)))^k$ and $\balpha \in [0,1]^L$, $f_{s_m}(\bTheta,\balpha)$ decreases monotonically to $f_{-\infty}(\bTheta,\balpha)$ as $m \to \infty$. Since $(\mathscr{C}(\phi_1(C)) \times \dots \times \mathscr{C}(\phi_L(C)))^k \times [0,1]^L$ is compact, appealing to Dini's Theorem \cite{apostol1964mathematical}, the result follows.
\end{proof}
For notational simplicity, let,  $\mathcal{M}_s(\bx,\bTheta,\balpha)=M_s( \sum_{l=1}^L \alpha_l \|\bx_i - \btheta_{1,l}\|^2 ,\dots, \sum_{l=1}^L \alpha_l \|\bx_i - \btheta_{k,l}\|^2 )$.
\begin{lem}\label{uslln 2}
(Uniform SLLN) Let $g_{\bTheta}(\bx)=\M_s(\bx,\bTheta,\balpha)$ and $\mathcal{G}=\{g_{\bTheta,\balpha}: \bTheta \in (\mathscr{C}(\phi_1(C)) \times \dots \times \mathscr{C}(\phi_L(C)))^k, \balpha \in [0,1]^L\}$. Then $\sup_{g \in \mathcal{G}}|\int g dP_n -\int g dP| \to 0$, almost everywhere $[P]$.
\end{lem}
\begin{proof}
It is enough to show that for any $\epsilon>0$, there exists $\mathcal{G}_\epsilon \subset \mathcal{G}$ such that $|\G_\epsilon|<\infty$ and for all $g \in \mathcal{G}$, there exist $\dot{g},\bar{g} \in \G_\epsilon$ with $\dot{g} \leq g \leq \bar{g}$ such that $\int(\bar{g}-\dot{g})dP < \epsilon$.
    
    We begin by observing that  $\mathcal{M}_s(\cdot,\cdot,\cdot)$ is a continuous function on the compact set $C \times (\mathscr{C}(\phi_1(C)) \times \dots \times \mathscr{C}(\phi_L(C)))^k \times [0,1]^L$, it is uniformly continuous by the Heine-Cantor theorem \cite{apostol1964mathematical}. 
    This implies that for any $\epsilon > 0$, we can choose $\delta_1$ and $\delta_2$, small enough such that  such that $\|\btheta_{j,l}-\btheta_{j,l}^\prime\|< \delta_1$ for all $j=1,\dots,k$, $l=1,\dots,L$ and $\|\balpha-\balpha\|_2 < \delta_2$, we have
    \begin{equation}
    \label{sap1}
    | \mathcal{M}_s(\bx,\bTheta,\balpha) - \mathcal{M}_s(\bx,\bTheta^\prime,\balpha^\prime) | < \frac{\epsilon}{2}.
    \end{equation}
    Compactness further implies that $\phi_l(C)$ is totally bounded for all $l=1,\dots,L$, so we may create a $\delta_1$-net $N_{\delta_1}^{(l)}$ of $\mathscr{C}(\phi_l(C))$. That is, $|N_{\delta_1}^{(l)}|< \infty$, and for all $\btheta \in \mathscr{C}(\phi_l(C))$ there exists $\btheta^\prime \in N_{\delta_1}^{(l)}$ such that $\|\btheta-\btheta^\prime\|<\delta_1$. Similarly we construct a $\delta_2$ net of compact set $[0,1]^L$, $S_{\delta_2}$. This means that  for all $\balpha \in [0,1]^L$, there exists $\balpha^\prime \in S_{\delta_2}$ such that $\|\balpha-\balpha^\prime\|<\delta_2$.
    Now, choose 
    $\G_\epsilon=\bigg\{\max\{\mathcal{M}_s(\bx,\bTheta,\balpha)+\pm \epsilon/2,0\}: \btheta_{j,1} \in N_{\delta_1}^{(1)},\dots,\btheta_{j,l} \in N_{\delta_1}^{(l)} \text{ for all } j=1,\dots,k \text{ and } \balpha \in S_{\delta_2} \bigg\}$. Now for any $\bTheta \in (\mathscr{C}(\phi_1(C)) \times \dots \times \mathscr{C}(\phi_L(C)))^k$ and $\balpha \in [0,1]^L$, let
    \begin{align*}
        \dot{g}_{\bTheta}(x) & =\max\{\mathcal{M}_s(\bx,\bTheta^\prime,\balpha^\prime)-\epsilon/2\}\\
        \bar{g}_{\bTheta}(x) & =\mathcal{M}_s(\bx,\bTheta^\prime,\balpha^\prime)+\frac{\epsilon}{2}.
    \end{align*}
    Here $\btheta_{j,l}^\prime \in N_{\delta_1}^{(l)}$ and $\|\btheta_{j,l}-\btheta_{j,l}^\prime\| < \delta_1$ for all $j=1,\dots,k$, $l=1,\dots,L$ and $\|\balpha-\balpha^\prime\|_2 < \delta_2$. By the construction of $\bTheta^\prime=\{\btheta_1^\prime,\dots,\btheta^\prime_k\}$ from Equation \eqref{sap1}, it follows that $\dot{g} \leq g \leq \bar{g}$. \,
    It remains to show that $\int(\bar{g}-\dot{g})dP < \epsilon$. To see this, 
\[\int(\bar{g}-\dot{g})dP \, \,  = \,  \int \bigg( M_s(\bx,\bTheta^\prime)+\frac{\epsilon}{2}-\max\big\{M_s(\bx,\bTheta^\prime)-\frac{\epsilon}{2},0\big\}\bigg) dP  \le   \epsilon \int dP \,\, = \,\, \epsilon.\]
\end{proof}
\begin{thm}\label{main theorem 2}
    (Strong Consistency) Under  A1 and A2, $\bTheta_{n,s} \overset{a.s.}{\to} \bTheta^\ast$ and $\balpha_{n,s} \overset{a.s.}{\to} \balpha^\ast$ as $n \to \infty$ and $s \to -\infty$.
\end{thm}
 \begin{proof}
     We must show for arbitrarily small $r>0$, the minimizer $\bTheta_{n,s}$ eventually lies inside the ball $B((\bTheta^\ast,\balpha^\ast),r)$. From A2, it suffices to show that for all $\eta>0$, there exists $N_1>0$ and $N_2<0$ such that $n>N_1$ and $s<N_2$ implies that $\Psi(\bTheta_{n,s},\balpha_{n,s},P) - \Psi(\bTheta^\ast,\balpha^\ast,P) \leq \eta$, almost everywhere $[P]$. 
    
    We observe that $\Psi(\bTheta_{n,s},\balpha_{n,s},P)-\Psi(\bTheta^\ast,\balpha^\ast,P) = \xi_1 + \xi_2 + \xi_3$, where
    \begin{align*}
    \xi_1 & = \Psi(\bTheta_{n,s},\balpha_{n,s}) - \int \M_s(\bx,\bTheta_{n,s},\balpha_{n,s})dP - \lambda H(\balpha_{n,s})\\
     \xi_2 & = \int \M_s(\bx,\bTheta_{n,s},\balpha_{n,s})dP - \int \M_s(\bx,\bTheta_{n,s},\balpha_{n,s})dP_n \\
     \xi_3 & = \int \M_s(\bx,\bTheta_{n,s},\balpha_{n,s})dP_n - \Psi(\bTheta^\ast,\balpha^\ast)+ \lambda H(\balpha_{n,s}).
    \end{align*}
    We first choose $N_2<0$ such that if $s<0$, $|\min_{\btheta \in \bTheta} \sum_{l=1}^L \alpha_l \|\phi(\bx)-\btheta_l\|^2 - \M_s(\bx,\bTheta,\balpha)| < \eta/6 $ for all $\bx \in C$ and $\bTheta \in (\mathscr{C}(\phi_1(C)) \times \dots \times \mathscr{C}(\phi_L(C)))^k$. This implies that 
    \begin{align*}
        &\xi_1 = \Psi(\bTheta_{n,s},\balpha_{n,s}) - \int \M_s(\bx,\bTheta_{n,s},\balpha_{n,s})dP -\lambda H(\balpha_{n,s}) \\
        & =\int \bigg(\min_{\btheta \in \bTheta} \sum_{l=1}^L \alpha_l \|\phi(\bx)-\btheta_l\|^2 - \M_s(\bx,\bTheta_{n,s},\balpha_{n,s}) \bigg) dP \leq \frac{\eta}{6} \int dP = \frac{\eta}{6}.
    \end{align*}
    Appealing to Lemma \ref{uslln 2}, we choose $N_1>0$ such that $n> N_1$ implies that $\xi_2< \eta/3$. To bound the third term $\xi_3$, we observe the following:
    \begin{align}
    \xi_3 & = \int \M_s(\bx,\bTheta_{n,s},\balpha_{n,s})dP_n - \Psi(\bTheta^\ast,\balpha^\ast) + \lambda H(\balpha_{n,s}) \nonumber\\
    & \le \int \M_s(\bx,\bTheta^\ast,\balpha^\ast)dP_n - \Psi(\bTheta^\ast,\balpha^\ast) + \lambda H(\balpha^\ast) \label{sap2} \\
    & \le \int \M_s(\bx,\bTheta^\ast,\balpha^\ast)dP - \Psi(\bTheta^\ast,\balpha^\ast) + \eta/6 \label{sap3}\\
    & \le \int \left\{\min_{\btheta \in \bTheta} \sum_{l=1}^L \alpha_l \|\phi(\bx)-\btheta_l\|^2 + \eta/6 \right\} dP - \int \min_{\btheta \in \bTheta} \sum_{l=1}^L \alpha_l \|\phi(\bx)-\btheta_l\|^2 dP +\eta/6 \label{sap4}\\
    & = \eta/3 \nonumber
    \end{align}
    Eq. \eqref{sap2} holds since $\bTheta_{n,s}$ is the minimizer for $\int \M_s(\bx, \bTheta,\balpha)dP+\lambda H(\balpha)$, and Eqs. \eqref{sap3} and \eqref{sap4} follow from Lemma \ref{uslln 2} and Theorem \ref{uniform}. Thus,  $$\Psi(\bTheta_{n,s},\balpha_{n,s})-\Psi(\bTheta^\ast,\balpha^\ast) = \xi_1 + \xi_2 + \xi_3 \le \eta/6 + \eta / 3+ \eta/3< \eta.$$ 
 \end{proof}

\section{Results and Performance}
All the experiments were undertaken in an Intel(R) Core(TM)i3-5010U 2.10 GHz processor, 4GB RAM, 64-bit Windows 8 OS in \texttt{R} programming language.
\begin{table}[h!]
 \caption{$p$-Values for Wilcoxon's Signed rank Test on Single Kernel Datasets}
    \label{tab:my_label}
    \centering
    \begin{tabular}{c c c c c}  
    \hline
         Dataset & Kernel Power $k$-means & Kernel $k$-means & Power $k$-means & Spectral Clustering \\
         \hline
       Yale  & --- & 0.0156 & $4.80 \times 10^{-5}$ & 0.0178\\
       JAFFE & --- & 0.0371 & 0.279 & 0.1675 \\
       TOX171 & 0.489 & 0.0024 & 0.0046 & ---\\
       Seeds & --- & 0.5174 & 0.3791 & 0.6844 \\
       Lung & --- & 0.0048 & $2.16 \times 10^{-10}$ & 0.0017\\
       Isolet & --- & 0.0024 & 0.0007 & 0.0476\\
       Lung Discrete & --- & $4.63 \times 10^{-7}$ & 0.0041 & 0.0068\\
       COIL20 & --- & 0.0001 & 0.0756 & 0.0349\\
       GLIOMA & --- & $1.45 \times 10^{-4}$ & 0.349 & $4.15 \times 10^{-6}$ \\ 
       \hline
    \end{tabular}
\end{table}

\section{Additional Experiments with ++ Initialization}
\begin{table}[ht]
     \caption{Average NMI values and average rank on real data; $+$ ($\approx$) indicates statistically significant (equivalent) result with respect to the best performing algorithm for that row.}
    \label{tab single}
    \centering
    \begin{tabular}{c c c c c}
    \hline
         Dataset & Kernel Power $k$-means++ & Kernel $k$-means++ & Power $k$-means++ & Spectral Clustering++ \\
         \hline
         Yale & \textbf{0.6324} & 0.5546 & 0.1764 & 0.5754\\
        JAFFE & \textbf{0.9246} & 0.8467 & 0.9074 & 0.8948\\
        TOX171 & \textbf{0.3946} &  0.2187 &  0.1931 &  0.3741\\
        Seeds & \textbf{0.7648} & 0.7156 & 0.7482 & 0.7382\\
        Lung & \textbf{0.6954}  &  0.5863 &  0.2196 &  0.5550 \\
        Isolet & \textbf{0.8672} & 0.7769 & 0.7812 & 0.8008\\
        Lung discrete & \textbf{0.8423} & 0.5825 & 0.6719 & 0.7349\\
        COIL20 & \textbf{0.8240} & 0.6913 & 0.7530 & 0.7264\\
        GLIOMA & \textbf{0.6412} & 0.4315 &  0.5903 & 0.2876 \\
         \hline
    \end{tabular}
\end{table}
In this section, we compare the peer algorithms when initiated using $k$-means++ seeding in the kernel space rather than random initializations. All the algorithms are seeded from the same initial centroids, chosen by a ++ seeding based on the distances in the kernel space, and run until convergence. This procedure is repeated 20 times and the average NMI vales are reported in Table~\ref{tab single}.  We see that the same trends are conveyed as the results in the maini text; in particular, Table~\ref{tab single} shows that the proposed KPK algorithm outperforms the other peer methods consistently. 












\end{document}